\documentclass[11pt]{amsart}

\usepackage[style=authoryear, backend=biber, natbib=true]{biblatex}
\usepackage{xurl}
\usepackage{hyperref}
\hypersetup{
    colorlinks=true,
    linkcolor=blue,
    filecolor=magenta,      
    urlcolor=cyan,
    pdftitle={MULTI-HEAD SELF ATTENTION IS A PARAMETER IDENTIFICATION MECHANISM},
    pdfpagemode=FullScreen,
}

\usepackage{amssymb}

\usepackage{graphicx} 

\newtheorem{thm}{Theorem}[section]
\newtheorem{prop}[thm]{Proposition}
\newtheorem{lem}[thm]{Lemma}
\newtheorem{cor}[thm]{Corollary}
\newtheorem{fact}[thm]{Fact}

\theoremstyle{definition}

\numberwithin{equation}{section}

\newcommand{\N}{\mathbb{N}}  

\newcommand{\R}{\mathbb{R}}  
\renewcommand{\Rn}[1]{\mathbb{R}^{#1}}  
\newcommand{\Rnxm}[2]{\mathbb{R}^{#1 \times #2}}  

\newcommand{\C}{\mathbb{C}}  

\renewcommand{\vec}[1]{\mathbf{#1}}  

\DeclareMathOperator{\Softmax}{Softmax} 
\newcommand{\SoftmaxCol}{\Softmax_{\mathrm{col}}}

\begin{document}
\title[Multi-Head Attention and Identification]{MULTI-HEAD SELF ATTENTION IS A PARAMETER IDENTIFICATION MECHANISM}

\author{W Ross Morrow}
\email{morrowwr@gmail.com}
\urladdr{wrossmorrow.com}
\urladdr{github.com/wrossmorrow}

\begin{abstract}
We prove that a multi-head scaled dot product attention can be viewed as a parameter identification strategy. The ratio of unidentified parameters to the total number of parameters scales like the reciprocal of the number of heads ($1/2 \to 1/(2H)$), meaning models with more heads are structurally more identified. A subtle side effect of the mathematics observation that attention can never be fully identified. Similarly we also show that some bias terms can have no effect on softmax-based attention layers in both the single- and multiple-head settings, though this is mostly a curiosity that should have a marginal effect on model size and model training/prediction efficiency. We also touch on modern improvements to transformers including RoPE and GQA from this perspective, illustrating how those as well can improve the ratio of ``meaningful'' parameters to all parameters. Simple numerical examples demonstrate that training can indeed involve updates that overlap model-invariant subspaces that arise from a lack of identification. As part of our experiments we use a ``rebalancing'' approach that can ``fix'' updates that overlap unindentified subspaces but do not try to present evidence this should actually be adopted. Instead we simply view our numerical results as exploring and confirming the theoretical results. As a whole we discuss a purely mathematical/statistical explanation, identification, for why specific architectural choices in transformers may have improved performance. 
\end{abstract}

\maketitle

\tableofcontents


\section{Introduction}

A fundamental part of the current success of transformer architectures is the multi-head attention introduced in \citet{vaswani2023}. This work popularized not only using scaled dot products as the basis of a mixing operator to train relevant context for token prediction, but also the branching strategy of using multiple low-rank heads in a batched fashion followed by ``averaging'' of their effects into the embedding dimension. The formal form of these interactions raises a question about {\em parameter identification}: whether all parameters (weights and biases) introduced into a modeling architecture actually represent trainable effects. While sidestepping such discussions is understandable and even probably practical in the face of the complexities of modern neural networks like transformers, using dot products in token mixing is a relatively obvious source of potential reduction in the effective predictive power of each individual parameter. 

In this short note we focus on and examine the interplay between identification and dot product attention. We observe how by multiplying weight matrices together we lose identification, both by demonstrating explicit invariances and identifying where directional derivatives vanish. Importantly these results show that it is impossible for {\em all} query-key value weights to have meaningful effects and, for similar parameter counts, a multi-head attention model has a higher fraction of identified parameters than a single-head attention. Specifically the unidentified fraction of parameters is related to the reciprocal of the number of heads ($1/2 \to 1/(2H)$, though there are other practical and competing consequences to increasing the number of heads). The observation of a symmetry that implies a reduction in identification is not in itself so surprising or novel. However we are unaware of the derived observation that splitting into multiple heads has the byproduct of (perhaps dramatically) improving identification, or the fact that this basic effect is also further enhanced by RoPE and GQA. 

We also undertake and report on some small scale and simple numerical experiments to suggest how this might affect model training in practice. Gradients with respect to weights will avoid symmetries that are connected to unidentified parameters by definition. Better optimizers add in higher order or momentum terms break this symmetry and can overlap invariant subspaces in directional derivatives. Our numerical experiments are mostly pointed at confirmation of the theoretical results, not to proposing specific methods to improve training. We do though present and use a ``rebalancing'' method as part of our confirmations that could be used to ``canonicalize'' weight iterates in training.

Our overall perspective is that identification is not just a theoretical curiosity, nor a classical statistical consideration outdated for deep learning. When parameters or parameter combinations are unidentified, we both store data on disk and in memory that doesn't influence outcomes as well as can train with updates to weights or biases that have reduced (or no) effect on model outputs. For effective and efficient training and inference, modelers should ideally aim to ensure that as many parameters as possible are actually influencing model predictions; in mathematical language we should aim for injective representations over the parameter space. 

In itself this is probably an uncontroversial if idealistic perspective. However architectural innovations in deep networks like transformers have been made that have the byproduct of improving parameter identification, without this being a directly stated motivation for those architectures. Presuming more identification indeed means more predictive capability for the same model size, at least in parameter counts, changing identification raises a confound for {\em what} is making more identified architectures better. ``Structural'' reasons to introduce multiple heads (or RoPE, or GQA) can be valid at the same time as predictive performance is improved simply because more of the parameters have predictive influence or meaning.

Any potential improvements in identification presented by multi-head attention are, surely, not the only mechanism driving model quality improvement. A multi-head attention setup is functionally distinct from a single-head attention, in the sense that these two generate distinct functional classes mapping one set of embedded token sequences to another. Along this line of thinking we also show there is no effective-parameter-comparable single-head attention we can build for a given multi-head attention without lifting to a larger embedding space for the entire model. 

This note starts with a short review of recent literature related to transformers and identification. We follow that with a careful definitions section to properly set the stage for results we derive. The third section focuses on our identification results, particularly that the architecture of multi-head attention alone improves identification. Our final substantive section reviews our numerical experiments to explore and confirm the theory. The note closes with brief review of the main conclusions we draw. 


\section{Related Literature}

Of course the basis for our discussion is the pivotal 2017-but-updated paper \citet{vaswani2023} introducing the main components of self attention and the transformer architecture. Here is the motivation for a multihead attention as stated there:
\begin{quotation}
[W]e found it beneficial to linearly project the queries, keys and values h times with different, learned
linear projections ... Multi-head attention allows the model to jointly attend to information from different representation subspaces at different positions. With a single attention head, averaging inhibits this.
\end{quotation}
Having effectively popularized the architecture enabling much of today's deep networks there have been extensive discussions about, investigations into, and extensions to this architecture. We'll review several related lines of enquiry here.

The author originally learned about transformers and worked on this note through experiments with \texttt{nanoGPT} \citep{karpathy2022}, reading a formal report from Google Brain \citep{phuong2022}, and of course \citet{vaswani2023}. By today our results are not entirely novel nor are really intended to be. \citet{zhao2025} discuss group symmetries in detail identifying a $GL_d(\R)$ symmetry in adjacent linear layers, including self-attention's query-key products, as well as $\Softmax$ invariance under certain translations. Both appear here with constructive proofs. \citet{zhao2025} in fact also elaborate on parameter identifiability and training dynamics, and cite numerous articles where the completeness of symmetry groups with respect to identifiability is known. Our contribution (if any) is to elaborate on how these results specifically relate to the introduction of multiple heads in attention as described in \citet{vaswani2023}, and how by doing so we improve parameter identification ``out of the box'' in a manner counfounding motivation from ``representationality''.

\citet{zhang2025}, motivated by well-known permutation symmetries in MLPs, study rotational symmetries in attention. This is a close subclass of the invariances we study: coarsely, and in our language, they focus on something like half the unidentified degrees of freedom we do. Importantly \cite{zhang2025} don't simply {\em observe} such symmetries but also propose a method to ``quotient out'' the associated invariances in training \citep{zhao2025}. We do something similar with weight rebalancing in AdamW though do not formally propose or analyze this rigorously as a serious method for training. 

Fundamentally multi-head attention achieves higher identification through lowering the rank of any specific attention applied, at least between the queries and keys. \citet{bhojanapalli2020} leverage the ``representation power'' argument for multiple low-rank heads and propose setting the attention head output projection size to the input sequence length (but experiment only within the uppser bound of the related embedding dimension). Their theoretical motivation is a representation theorem stating that to get an arbitrary stochastic output matrix out of an attention head one must choose an inner head dimensionality (here denoted $D$) at least the sequence length (here $E$). They run experiments with multi-head attentions with ``inflated'' output dimensions larger than the layer's embedding dimension divided by the number of heads as in \citet{vaswani2023}. Our results are compatible with improved performance in this architecture and also present a tradeoff: the ``most identified'' extreme of rank-1 heads ($D = 1$) is the least representationally powerful, but the ``most representational'' extreme of full-rank heads ($D = E$) is the most wasteful with parameters, and the original proposal ($D = E/H$ where $H$ divides $E$) effectively trades off these two effects of representational rank and parameter efficiency. 

As of writing many successful transformer models for language also adopt Rotary Positional Encodings (RoPE) following the work in \citet{su2023}. Importantly this technique partially breaks the symmetries in the original multi-head attention. We study a sketch of RoPE, concluding this related architecture again significantly improves identification (``$D^2 \to D$ identified'') by effectively inducing a positional family commutativity constraint. Identification, though, is seemingly not why RoPE was proposed. Much like multi-head attention RoPE appears to have a representational motivation with a very similar side effect on parameter efficiency. There is also recent interest in no positional encoding (NoPE) for learned long-context positional relationships \citep{wang2024} which would recenter the ``$GL_D(\R)$'' idenfication considerations we discuss.

Further enhancements like grouped query attention (GQA) \citep{ainslie2023} or cross attention \citep{brandon2024} have also appeared. With the exception of RoPE or any other modifier ``inside'' query-key products these would change the quantitative, not qualitative, results. Like RoPE we review GQA specifically, and how it can tie together head symmetries to fewer groups to improve identification. 

As a final note related to existing literature, we could also consider any relationship identification may have with the Lottery Ticket Hypothesis (LTH) \citep{frankle2019}. Experiments in that work suggest far fewer meaningful parameters than we formally show are identified. Specifically our results suggest for some characterizations of now-plain multi-head attention networks about 90\% of parameters have ``meaning'', whereas the LTH results are more like 10-20\%. This misalignment means {\em should} a lack of identification alone be a prominent feature in the LTH many more sources are missing. 


\section{Definitions}

Let’s say we have inputs $\vec{x} \in \{1, \dotsc , T\}^L \subset \Rn{L}$ where $L \in \N$ is a uniform token sequence length and $\{1, \dotsc , T\} \subset \N$ is the token vocabulary (or indices of them), and embed tokens as vectors in $\Rn{E}$ for some $E \in \N$. We focus on attention heads motivated by the \citet{vaswani2023}-popularized scaled dot product formulation stylized as
\begin{equation*}
	\vec{V} \Softmax \left( \vec{K}^\top\vec{Q} \right)
\end{equation*}
of attention with query-key-value structure as a basic layer unit. Note here we aren’t using the row-centric standards common in deep learning, in which some “broadcast” logic is more natural, but rather a maybe more mathematical column vector orientation. This is an equivalent representation with the right transposing that is a bit more consistent with general linear algebra knowledge but mainly just more comfortable for the author. We also exclude a scaling factor, as this is a computational convenience with no unique impact on theory, as we can just absorb scaling into the weights for our purposes. We describe the details of how this layer works in this section to provide us with the foundation for analysis.

Just for clarity, denoting a generic \texttt{pytorch}-style linear layer as 
\begin{equation*}
	\mathcal{L}(\vec{X}|\vec{W},\vec{b}) = \vec{X}\vec{W}^\top + \vec{1}\vec{b}^\top
\end{equation*}
we would implement
\begin{equation*}
	\Softmax_d \left( \mathcal{L}_Q(\vec{X})\mathcal{L}_K(\vec{X})^\top \right) \mathcal{L}_V(\vec{X})
\end{equation*}
The literal inputs $\vec{X}$ could differ or be batched (or both), as long as the dimensions work out. Perhaps importantly though the \texttt{pytorch} form has exactly the same inner matrix products we identify invariances in and require no modifications to understand. 

\subsection{Embeddings.}

Our first step, token embedding, actually occurs before attention and any layers: we are given a “table” $\vec{e} : \{1, \dotsc , T\} \to \Rn{E}$ that maps tokenized sequences into a vector space for some $E \in \N$, and construct a matrix representation of the input sequence as 
\begin{equation*}
	\vec{X} = \begin{pmatrix}
		\vert & & \vert \\
		\vec{e}(x_1) & \dotsb & \vec{e}(x_L) \\
		\vert & & \vert \\
	\end{pmatrix} \in \Rnxm{E}{L}
\end{equation*}
This embedding can be trained, or it can provide a level of pre-``trained'' semantic similarity between tokens for words or word parts. Regardless we generally consider inputs $\vec{X} \in \Rnxm{E}{L}$ of embedded tokens throughout layers. There are other properties satisfied by inputs to attention heads, such as layer normalization, that we won’t consider here. 

\subsection{Queries, Keys, and Values.}

With values in embedding space and choosing some $D \in \N$ we compute “queries” “keys” and “values” $\vec{Q}, \vec{K}, \vec{V}$ from linear (possibly affine) layers: 
\begin{align*}
	&\vec{Q}(\vec{X}) = \vec{W}_Q \vec{X} \in \Rnxm{D}{L} \\
	&\vec{K}(\vec{X}) = \vec{W}_K \vec{X} \in \Rnxm{D}{L} \\
	&\vec{V}(\vec{X}) = \vec{W}_V \vec{X} \in  \Rnxm{E}{L} \\
	&\text{where} \;\; \vec{W}_Q,\vec{W}_K \in  \Rnxm{D}{E} , \vec{W}_V \in \Rnxm{E}{E}
\end{align*}
Later we’ll consider bias terms, but for now we ignore them. Note also that 
\begin{equation*}
	\vec{Q}(\vec{X})
	= \vec{W}_Q\vec{X}
	= \begin{pmatrix}
		\vert & & \vert \\
		\vec{W}_Q\vec{x}_1 & \dotsb & \vec{W}_Q\vec{x}_L \\
		\vert & & \vert \\
	\end{pmatrix} \in \R^{E \times L}
\end{equation*}
(similarly for $\vec{K}(\vec{X})$ and $\vec{V}(\vec{X})$) and thus column ordering in queries, keys, and values is preserved relative to positional ordering in any original input sequence $\vec{x}$ (prior to embedding to lift up to $\vec{X}$ which also preserves ordering). Effectively, the keys and queries “project” embeddings into $\Rn{D}$ for some $D$, whereas the values (in a single head) are functionally a post-attention linear layer. 

We've assumed the value layer is ``square'' ($E \times E$) instead or rectangular, and ignore biases here which are easy to add back. The setup described here has $2DE + E^2$ parameters. If $D = E$ we have $3E^2$ parameters but the queries and keys can be low rank with $D < E$. These symbols and counts will get slightly more complicated for layers that ``lift'' or ``contract'' the embedding dimension, but not significantly meaningful way.

\subsection{Probabilities.}

Using a query-key multiplication as above is sufficient to effect ``embedded token mixing'' interpreted as any operation across tokens in a sequence. More commonly, transformer architectures also take a $\Softmax$ over columns of the query-key product matrix: 
\begin{equation*}
	\SoftmaxCol \left( \vec{K}(\vec{X})^\top \vec{Q}(\vec{X})\right) \in \Rnxm{L}{L}
\end{equation*}
again ignoring scaling, forming a set of column-specific probabilities. For ``causal'' modeling we would also include a mask whose effect is to make the $\Softmax$ respect ordering in the token embeddings. Formally, the $(d,\ell)$th element is defined as 
\begin{equation*}
	\SoftmaxCol \left( \vec{K}(\vec{X})^\top \vec{Q}(\vec{X})\right) _{d,l}
		= \frac{ \exp\{\vec{k}_d^\top\vec{q}_\ell\}  }{ \sum_c \exp\{\vec{k}_c^\top\vec{q}_\ell\} }
\end{equation*}
for some sum over $c$. The result is a left-stochastic matrix, i.e. one with positive entries whose column sums are all one. 

Left multiplication by $\vec{X}$ forms convex combinations of the columns of $\vec{X}$ in the result: 
\begin{equation*}
	\begin{pmatrix}
		\vert & \vert & & \vert \\
		\sum_{c=1}^L p_{c,1} \vec{x}_c & \sum_{c=1}^L p_{c,2} \vec{x}_c & \dotsb & \sum_{c=1}^L p_{c,L} \vec{x}_c \\
		\vert & \vert & & \vert \\
	\end{pmatrix}
\end{equation*}
where for simplicity we denote the $\SoftmaxCol$ results as $p$ (for “probability” of course). A “causal” attention, wherein we “mask” out any “future” embeddings when computing probabilities, results in a form 
\begin{equation*}
	\begin{pmatrix}
		\vert & \vert & & \vert \\
		\vec{x}_1 & \sum_{c=1}^2 p_{c,2} \vec{x}_c & \dotsb & \sum_{c=1}^L p_{c,L} \vec{x}_c \\
		\vert & \vert & & \vert \\
	\end{pmatrix}
\end{equation*}
embodied by a stochastic upper triangular $\SoftmaxCol$ matrix, with entries depending only on forward relationships derived from sequence ordering. We’ll leave masking out of the notation most of the time, as what we discuss will not depend on masking. 

Finally, note we are not considering “cross attentions” wherein the queries and keys are determined by different inputs, as might be useful in translation oriented encoder-decoder networks. Formally, we’re focusing on “self attention” layers, though our basic results about identification in matrix products hold for cross attention as well because they don't relate to the inputs.

\subsection{Value Weighting.}

We then compose those probabilities with a linear layer using $\vec{V}(\vec{X})$, 
\begin{equation*}
	\vec{V}(\vec{X}) \SoftmaxCol \left( \vec{K}(\vec{X})^\top \vec{Q}(\vec{X})\right)
		= \vec{W}_V\vec{X} \SoftmaxCol \left( \vec{K}(\vec{X})^\top \vec{Q}(\vec{X})\right)
\end{equation*}
We could view the attention head as a layer computing 
\begin{equation*}
	\vec{X} \SoftmaxCol \left( \vec{X}^\top \vec{W}_K^\top \vec{W}_Q \vec{X} \right)
\end{equation*}
with the convex combinations mentioned above followed by a linear layer with weights $\vec{W}_V$, where each column is the “expected” embedding for each token in the sequence using the probabilities from the $\SoftmaxCol$.

\subsection{Multi-Head Attention.}

The \citet{vaswani2023}-popularized $H$-head attention generalizes the above choosing $H,D,C \in \N$ and taking 
\begin{align*}
	&\vec{Q}_h(\vec{X}) = \vec{W}_Q^h \vec{X} \in \Rnxm{D}{L} \\
	&\vec{K}_h(\vec{X}) = \vec{W}_K^h \vec{X} \in \Rnxm{D}{L} \\
	&\vec{V}_h(\vec{X}) = \vec{W}_V^h \vec{X} \in  \Rnxm{E}{L} \\
	&\text{where} \;\; \vec{W}_Q^h,\vec{W}_K^h \in  \Rnxm{D}{E} , \vec{W}_V^h \in \Rnxm{C}{E}
\end{align*}
for $h = 1, \dotsc, H$ and adding an additional affine operation
\begin{equation*}
	\sum_{h=1}^H \vec{W}_O^h \vec{V}^h(\vec{X}) \SoftmaxCol \left( \vec{K}_h(\vec{X})^\top \vec{Q}_h(\vec{X}) \right)
		+ \vec{b}_O \vec{1}^\top
\end{equation*}
for parameters $\vec{W}_O^h \in \Rnxm{E}{C}, \vec{b}_O \in \Rn{E}$ and $\vec{1} \in \Rn{L}$ a vector of ones. This form is the same as “concatenate and project” common in the literature and code like nanoGPT: 
\begin{equation*}
	\begin{pmatrix}
		\vec{W}_O^1 & \dotsb & \vec{W}_O^H
	\end{pmatrix}
	\begin{pmatrix}
		\vec{V}^1(\vec{X}) \SoftmaxCol \left( \vec{K}_1(\vec{X})^\top \vec{Q}_1(\vec{X}) \right) \\
		\vdots \\
		\vec{V}^H(\vec{X}) \SoftmaxCol \left( \vec{K}_H(\vec{X})^\top \vec{Q}_H(\vec{X}) \right) \\
	\end{pmatrix}
	+ \vec{b}_O \vec{1}^\top
\end{equation*}
Note there is a fan-out/fan-in effect between value and output weighting, with strong similarity to adding up low-rank outer products. We project ``down'' to $\Rn{C}$ from $\Rn{E}$ with the head-specific value weights (presuming $C < E$), and ``lift'' back up to $\Rn{E}$ with the “output” weights introduced enabling us to ``pool'' effects from the different heads. 

Counting up we have $2HE(D + C)$ weights, A convenient choice for the dimensions here is $D = C = E/H$ choosing an $H$ that divides $E$ which would imply we have a total of $4HED = 4E^2$ weights. 


\section{Parameter Identification in Attention}

Executing matrix-matrix products between weights results in a loss of degrees of freedom identifiable through training. Specifically, when our model multiplies weights directly inside attention heads, we form sums of products of elements that will be invariant to specific perturbations in the individual weights. This means two things: the model is ``over-specified'' in the specific sense that it has more parameters than degrees-of-freedom or ``uniquely meaningful'' effects, and there are linear subspaces over which layers including weight matrix products have vanishing directional derivatives with respect to weights (for any observations). This section refines these ideas and proves our basic results: (i) there are $D^2$ unidentified degrees of freedom in commonly used attention heads, and (ii) using multiple heads improves (increases) the ratio of identified degrees of freedom to parameters. Note that if there are always $D^2$ unidentified parameters (or combinations), an attention head can never be fully identified only asymptotically well identified. 

A side note about our style of discourse here: we purposefully use probably overly deliberate and constructive derivations. Hopefully that doesn't obscure our basic points which could be stated more tersely. 

\subsection{A Trivial, Yet Illustrative, Example.}

Suppose we forgo the embedding. That is, set $E = 1$ and $\vec{e}(x) = x$ in which case $\vec{X} = \vec{x}^\top \in \Rnxm{1}{L}$. Then $\vec{W}_Q,\vec{W}_K \in \Rnxm{D}{1}$ are column vectors, the queries and keys 
\begin{align*}
	&\vec{Q}(\vec{x}) = \vec{W}_Q \vec{x}^\top \in \Rnxm{D}{L} \\
	&\vec{K}(\vec{x}) = \vec{W}_K \vec{x}^\top \in \Rnxm{D}{L}
\end{align*}
are rank-1 transformations (outer products), and 
\begin{align*}
	\vec{K}(\vec{x})^\top \vec{Q}(\vec{x})
		= \vec{x}^\top \vec{W}_K^\top \vec{W}_Q \vec{x}
		= \omega \vec{xx}^\top 
	\quad\quad
	\omega = \vec{W}_K^\top \vec{W}_Q \in \R
\end{align*}
Even though this is entirely artificial (especially as we would likely constrain $D \leq E$) it raises an important point: for all the $2D$ variables we admit with $\vec{W}_Q,\vec{W}_K \in \Rnxm{D}{1}$ we only have a single effect, the inner product of these two vectors. Even if $D = 1$, we model with two numbers but only their product can have an effect. This is the prototype of an “identification” problem, in that we have $2D-1$ out of $2D$ "trainable" values that can't influence the layer output. 

\subsection{Invariances in Query-Key Products.}

In general a lack of parameter identification arises from any query-key products as defined here. The fundamental operation is 
\begin{equation*}
	\vec{K}(\vec{X})^\top\vec{Q}(\vec{X}) = \vec{X}^\top\vec{W}_K^\top \vec{W}_Q \vec{X}
\end{equation*}
which we’ll sometimes refer to as a “quadratic form” even though that’s not technically precise to the usual meaning. \citet{elhage2021} call this the ``QK-circuit''. The now-eponomous ``key-value cache'' \citep{brandon2024} for efficient inference stores results of the layer application, possibly ``hiding'' some of these effects. 

Multiplying two matrices has easily discoverable, trivial, invariances: First, any shared row permutation defined by $\vec{P} \in \Rnxm{D}{D}$, $\vec{P}^\top\vec{P} = \vec{I}$, can be injected: 
\begin{equation*}
	\left( \vec{PW_K} \right)^\top \left( \vec{PW}_Q \right)
		= \vec{W_K}^\top \vec{P}^\top\vec{P} \vec{W}_Q
		= \vec{W_K}^\top \vec{W}_Q
\end{equation*}
Of course this also shows any orthonormal matrix can be injected in the same way (of which permutation matrices are just a special case). Also any nonsingular row scaling $\vec{A} = \mathrm{diag}(\vec{a})$ for any vector $\vec{a} \in \Rn{D}$ with no zeros is an invariant via
\begin{equation*}
	\left( \vec{A}^{-1}\vec{W_K} \right)^\top \left( \vec{A}\vec{W}_Q \right)
		= \vec{W_K}^\top \vec{A}^{-\top}\vec{A} \vec{W}_Q
		= \vec{W_K}^\top \vec{A}^{-1}\vec{A} \vec{W}_Q
		= \vec{W_K}^\top \vec{W}_Q
\end{equation*}
We've more or less shown the following general invariance: 
\begin{lem}
For any pair $(\vec{W}_Q,\vec{W}_K) \in \Rnxm{D}{E}$ and any nonsingular $\vec{S} \in \Rnxm{D}{D}$, the pair $(\vec{SW}_Q,\vec{S}^{-\top}\vec{W}_K)$ creates the same quadratic form.
\end{lem}
\begin{proof}
$(\vec{S}^{-\top}\vec{W}_K)^\top(\vec{SW}_Q) = \vec{W}_K\vec{S}^{-1}\vec{S}\vec{W}_Q = \vec{W}_K\vec{W}_Q$
\end{proof}
Our examples above the lemma furnish specific examples: we can swap the rows of our weights around without effect, or more generally rotate them in any way, or scale up one set of weights as long as we comparably shrink the other. This fact is corroborated by an analysis of directional derivatives, which is actually a distinct result about locally linear approximations relevant to training. Generally speaking, this invariance should be a reasonable concern for gradient style optimization methods, unless there are some other (natural or imposed) corrective mechanisms at play. 

This also means there are $D^2$ out of $2DE$ “unidentified” degrees of freedom in any $D \times E$ attention head. That is, only a fraction
\begin{equation*}
	\frac{2DE - D^2}{2DE} = \frac{2E - D}{2E} = 1 - \frac{1}{2} \frac{D}{E}
\end{equation*}
of the parameters have “uniquely meaningful” effects, though we can’t necessarily find "the specific ones". If we used a full rank attention head with $D = E$, only $1/2$ of the parameters are “uniquely meaningful” in this sense; but note that the smaller $D$ is, i.e. the lower rank the query-key transformation is, the more identification we can obtain. But it can never be “all of the parameters”, because the best we can do is $1 - 1/(2E)$ (with $D = 1$). 

We should comment that it is hard to envision how one could directly and effectively capitalize on this fact by training over only “uniquely meaningful” parameters. This point hopefully gets clearer with the additional exposition below. 

\subsection{A “Principled” Approach.}

We can view this effect a different if laborious way: 
\begin{lem}
For any $\vec{W}_Q,\vec{W}_K \in \Rnxm{D}{E}$ there are matrices $\bar{\vec{U}}_Q, \bar{\vec{U}}_K \in \Rnxm{E}{D}$ with orthonormal columns and a vector $\boldsymbol{\sigma} \in \Rn{D}$ with non-negative entries such that 
\begin{equation*}
	\bar{\vec{W}}_Q = \mathrm{diag}(\boldsymbol{\sigma})\bar{\vec{U}}_Q
	\quad,\quad
	\bar{\vec{W}}_K = \mathrm{diag}(\boldsymbol{\sigma})\bar{\vec{U}}_K
\end{equation*}
has the same quadratic form as does $(\vec{W}_Q,\vec{W}_K)$
\end{lem}
\begin{proof}
Take an SVD of the (square) product $\vec{W}_K^\top\vec{W}_Q = \bar{\vec{U}}\bar{\boldsymbol{\Sigma}}\bar{\vec{V}}^\top$ where $\bar{\vec{U}},\bar{\vec{V}} \in \Rnxm{E}{E}$ have orthonormal columns and $\bar{\boldsymbol{\Sigma}} \in \Rnxm{E}{E}$ is a diagonal matrix with non-negative elements non-decreasing in index the last $E-D$ of which must be zero. The result follows immediately taking $\bar{\vec{U}}_Q = \bar{\vec{U}}_{1:D,:}$, $\bar{\vec{U}}_K = \bar{\vec{V}}_{1:D,:}$ where $1:D,:$ notates taking first $D$ rows and $\sigma_d^2 = \bar{\boldsymbol{\Sigma}}_{d,d}$ for $d=1,\dotsc,D$. 
\end{proof}
Note that the resulting “factorization” here is still signed-permutation invariant. That is, we can change signs and ordering without influencing the result product. However neither of those are linear effects, and thus perhaps not relevant from the perspective of gradient-based training. 

Imagine for a moment we used this as a {\em definition} of $\vec{W}_Q,\vec{W}_K$. We would train over $2DE + D$ parameters while constraining the matrices to be composed of orthonormal columns. Putting aside the practicality of doing this, we effectively reduce the number of degrees of freedom in each matrix by 
\begin{equation*}
	\frac{D(D-1)}{2} + D = \frac{D(D+1)}{2}
\end{equation*}
from the orthonormality constraints on the $D$ columns. In total, we would have $2DE + D$ parameters with only 
\begin{equation*}
	2DE + D - 2\frac{D(D+1)}{2} 
		= 2DE + D - D(D+1)
		= 2DE - D^2
\end{equation*}
actual degrees of freedom. Note the same outcome: $D^2$ fewer degrees of freedom. 

Actually following through on this thought experiment, probably by keeping the query-key weights orthonormal and adding scaling values as an other term in the product, is likely impractical. Projected gradient methods might, in principle, allow us to maintain orthonormality over weight up dates but at the cost of 2 SVDs per attention layer update. Literally executing repeated factorizations of that sort during training would not scale. 

\subsection{Directional Derivatives.}

Unsurprisingly these invariances overlap with differentiation. Our lemmas above formally apply to direct perturbations of weights but not local linearity of the query-keyquery-key quadratic form, which might behave differently. Given that training updates are closely tied to differentiation, it’s worth considering the same sort of effect from that viewpoint. 

The $(\boldsymbol{\delta}\vec{W}_Q,\boldsymbol{\delta}\vec{W}_K)$-directional derivative at $(\vec{W}_Q,\vec{W}_K)$ of the quadratic form in attention is 
\begin{equation*}
	\vec{X}^\top \left( \vec{W}_K^\top\boldsymbol{\delta}\vec{W}_Q + \boldsymbol{\delta}\vec{W}_K^\top\vec{W}_Q \right) \vec{X}
\end{equation*}
Such a derivative vanishes for any embedded token sequence (or it’s propagation through a network) when
\begin{equation*}
	\vec{W}_K^\top\boldsymbol{\delta}\vec{W}_Q + \boldsymbol{\delta}\vec{W}_K^\top\vec{W}_Q = \vec{0}
\end{equation*}
For which subspaces of $(\boldsymbol{\delta}\vec{W}_Q,\boldsymbol{\delta}\vec{W}_K)$ does that hold? One trivial zero is obvious: 
\begin{equation*}
	(\boldsymbol{\delta}\vec{W}_Q,\boldsymbol{\delta}\vec{W}_K) = (\vec{W}_Q,-\vec{W}_K)
\end{equation*}
as can be immediately checked. But we should ask for which other choices that could hold, and how many lost “degrees of freedom” that implies.
\begin{lem}
Let $D \leq E$ and presume $\vec{W}_Q,\vec{W}_K \in \Rnxm{D}{E}$ are full-rank (that is, are rank $D$). There are directions $(\boldsymbol{\delta}\vec{W}_K,\boldsymbol{\delta}\vec{W}_Q) \in \Rnxm{D}{E}$ derived from linear transformations of any $\vec{S} \in \Rnxm{D}{D}$ in which $\vec{W}_K^\top\vec{W}_Q$ has zero
directional derivative.
\end{lem}
Note here we do not require $\vec{S}$ to be invertible, unlike the result above.
\begin{proof}
Take (wide) SVDs
\begin{equation*}
	\vec{W}_K = \bar{\vec{U}}_K\bar{\boldsymbol{\Sigma}}_K \bar{\vec{V}}_K^\top
	\quad\quad
	\vec{W}_Q = \bar{\vec{U}}_Q\bar{\boldsymbol{\Sigma}}_Q \bar{\vec{V}}_Q^\top
\end{equation*}
where $\bar{\vec{U}}_K, \bar{\boldsymbol{\Sigma}}_K \in \Rnxm{D}{D}$ and $\bar{\vec{V}}_K \in \Rnxm{E}{D}$ with $\bar{\vec{U}}_K^\top\bar{\vec{U}}_K = \vec{I}$ and $\bar{\boldsymbol{\Sigma}}_K$ diagonal and also invertible (the same holds for the $Q$ subscript). For any $\vec{S} \in \Rnxm{D}{D}$ define
\begin{equation*}
	\boldsymbol{\delta}\vec{W}_Q(\vec{S}) = \bar{\vec{U}}_K\bar{\boldsymbol{\Sigma}}_K^{-1} \vec{S} \bar{\vec{V}}_Q^\top
	\quad,\quad
	\boldsymbol{\delta}\vec{W}_K(\vec{S}) = - \bar{\vec{U}}_Q\bar{\boldsymbol{\Sigma}}_Q^{-1} \vec{S}^\top \bar{\vec{V}}_K^\top
\end{equation*}
Then
\begin{align*}
	\vec{W}_K^\top \boldsymbol{\delta}\vec{W}_Q(\vec{S})
		= \bar{\vec{V}}_K \bar{\boldsymbol{\Sigma}}_K \bar{\vec{U}}_K^\top \bar{\vec{U}}_K\bar{\boldsymbol{\Sigma}}_K^{-1} \vec{S} \bar{\vec{V}}_Q^\top
		= \bar{\vec{V}}_K \vec{S} \bar{\vec{V}}_Q^\top
\end{align*}
and 
\begin{align*}
	\boldsymbol{\delta}\vec{W}_K(\vec{S})^\top \vec{W}_Q
		= - \bar{\vec{V}}_K \vec{S} \bar{\boldsymbol{\Sigma}}_Q^{-1} \bar{\vec{U}}_Q^\top \bar{\vec{U}}_Q \bar{\boldsymbol{\Sigma}}_Q\bar{\vec{V}}_Q^\top
		= - \bar{\vec{V}}_K \vec{S} \bar{\vec{V}}_Q^\top
\end{align*}
and thus 
\begin{align*}
	\vec{W}_K^\top \boldsymbol{\delta}\vec{W}_Q(\vec{S}) + \boldsymbol{\delta}\vec{W}_K(\vec{S})^\top \vec{W}_Q = \vec{0}
\end{align*}
This proves the claim that the directional derivatives of $\vec{W}_K^\top\vec{W}_Q$ are zero in the directions $(\boldsymbol{\delta}\vec{W}_K(\vec{S}),\boldsymbol{\delta}\vec{W}_Q(\vec{S}) )$, for any $\vec{S}$, for all inputs $\vec{X}$. 
\end{proof}

The following addition is trivially verifiable:
\begin{cor}
The set of perturbations $\vec{S} \to (\boldsymbol{\delta}\vec{W}_Q(\vec{S}),\boldsymbol{\delta}\vec{W}_K(\vec{S}))$ that have zero directional derivative derived in the proof of the lemma is a linear subspace of dimension $D^2$. 
\end{cor}

Also, our trivial zero is included in the subspace of zeros:
\begin{cor}
$(\vec{W}_Q,-\vec{W}_K)$ is in the subspace of zeros of the directional derivative of the quadratic form derived above.
\end{cor}
\begin{proof}
Set $\vec{S} = \bar{\boldsymbol{\Sigma}}_K \bar{\vec{U}}_K^\top \bar{\vec{U}}_Q \bar{\boldsymbol{\Sigma}}_Q$ and verify 
\begin{align*}
	\boldsymbol{\delta}\vec{W}_K(\vec{S}) 
		&= \bar{\vec{U}}_K\bar{\boldsymbol{\Sigma}}_K^{-1} \bar{\boldsymbol{\Sigma}}_K \bar{\vec{U}}_K^\top \bar{\vec{U}}_Q \bar{\boldsymbol{\Sigma}}_Q \bar{\vec{V}}_Q^\top
		\\
		&= \bar{\vec{U}}_K \bar{\vec{U}}_K^\top \bar{\vec{U}}_Q \bar{\boldsymbol{\Sigma}}_Q \bar{\vec{V}}_Q^\top
		= \bar{\vec{U}}_Q \bar{\boldsymbol{\Sigma}}_Q \bar{\vec{V}}_Q^\top = \vec{W}_Q
		\\
	\boldsymbol{\delta}\vec{W}_Q(\vec{S})
		&= - \bar{\vec{U}}_Q\bar{\boldsymbol{\Sigma}}_Q^{-1} ( \bar{\boldsymbol{\Sigma}}_K \bar{\vec{U}}_K^\top \bar{\vec{U}}_Q \bar{\boldsymbol{\Sigma}}_Q )^\top \bar{\vec{V}}_K^\top 
		\\
		&= - \bar{\vec{U}}_Q\bar{\boldsymbol{\Sigma}}_Q^{-1}  \bar{\boldsymbol{\Sigma}}_Q \bar{\vec{U}}_Q^\top \bar{\vec{U}}_K \bar{\boldsymbol{\Sigma}}_K \bar{\vec{V}}_K^\top 
		\\
		&= - \bar{\vec{U}}_Q \bar{\vec{U}}_Q^\top \bar{\vec{U}}_K \bar{\boldsymbol{\Sigma}}_K \bar{\vec{V}}_K^\top 
		= - \bar{\vec{U}}_K \bar{\boldsymbol{\Sigma}}_K \bar{\vec{V}}_K^\top = - \vec{W}_K
\end{align*}
to prove the result.
\end{proof}

Importantly, this result does not say the {\em actual value} of the quadratic form is invariant for steps in this direction, but rather only that it is a higher-than-first-order effect. In particular, the actual change in the quadratic form is determined by characteristically messy second order effects in these directions:
\begin{align*}
	&\vec{W}_K^\top\vec{W}_Q 
		+ \vec{W}_K^\top \boldsymbol{\delta}\vec{W}_Q(\vec{S}) 
		+ \boldsymbol{\delta}\vec{W}_K(\vec{S})^\top \vec{W}_Q
		+ \boldsymbol{\delta}\vec{W}_K(\vec{S})^\top \boldsymbol{\delta}\vec{W}_Q(\vec{S}) 
	\\
	&\quad\quad
		= \vec{W}_K^\top\vec{W}_Q + \boldsymbol{\delta}\vec{W}_K(\vec{S})^\top \boldsymbol{\delta}\vec{W}_Q(\vec{S}) 
		\quad\quad\text{(for the linearly invariant steps)}
	\\
	&\quad\quad
		= \vec{W}_K^\top\vec{W}_Q 
			+ \bar{\vec{V}}_K\vec{S}\bar{\boldsymbol{\Sigma}}_Q^{-1}\bar{\vec{U}}_Q^\top
				\bar{\vec{U}}_K\bar{\boldsymbol{\Sigma}}_K^{-1} \vec{S} \bar{\vec{V}}_Q^\top
\end{align*}

\subsection{Multi-Head Attention Improves Identification.}

The results above apply immediately to each head in a multi-head attention and imply our main result. Suppose we split into $H$ heads where $H$ divides $E$ and set $D = E/H$. This is, in fact, the basic idea in \citet{vaswani2023} as well as implementations like \texttt{nanoGPT}. By the same logic as following our invariance lemma, there are $2HDE = 2E^2$ parameters in all the attention head query-key products (that is, not including value and output pooling weights), of which $HD^2 = E^2/H$ are unidentified, $D^2$ from each head. Equivalently we have a fraction
\begin{equation*}
	1 - \frac{HD^2}{2HDE} 
		= 1 - \frac{D}{2E} 
		= 1 - \frac{E}{2HE}
		= 1 - \frac{1}{2H}
\end{equation*}
identified parameters or degrees-of-freedom. Hence we can increase the identified fraction of parameters just by using more heads. Ignoring all other considerations, we can use up to E rank-one heads and again obtain “meaning” for a fraction $1 - 1/(2E)$ of parameters at most. That is likely a poor actual strategy due to the computationally intensive nature of a $\Softmax$-based attention head, and some batching ($D > 1$) is probably valuable. 

Importantly it shouldn’t take a large number of heads to see real benefits. For example with 6 heads $11/12$ or about 83\% of the parameters have identified effects, and with 8 heads already $15/16$ or about 94\% of the parameters have identified effects. 

Again this result doesn't directly account for the output projections involving the multiplication of $\vec{W}_O^h$ and $\vec{W}_V^h$ weights. But we again have weight matrix products like $\vec{W}_O^h\vec{W}_V^h$ which, even if low rank, have invariances that reduce degrees of freedom. This is the ``OV-Circuit'' in \citet{elhage2021}, and exactly the same situation occurs for these matrix products \citep{zhao2025}: 

\begin{lem}
For any pair $\vec{W}_O \in \Rnxm{E}{C}$ and $\vec{W}_V \in \Rnxm{C}{E}$ and any nonsingular $\vec{S} \in \Rnxm{C}{C}$, the pair $(\vec{W}_O\vec{S}^{-1}, \vec{SW}_V )$ has the same product as $(\vec{W}_O,\vec{W}_V)$. Moreover the directional derivatives of the product $\vec{W}_O\vec{W}_V$ with respect to directions in both weight matrices vanish on a subspace of dimension $C^2$. 
\end{lem}
\begin{proof}
The first claim is trivial. The second is proved with SVDs in the same manner as before, with a minor notational change: 
\begin{equation*}
	\boldsymbol{\delta}\vec{W}_O = \bar{\vec{U}}_O\vec{S}\bar{\boldsymbol{\Sigma}}_V^{-1}\vec{U}_V^\top
	\quad\quad
	\boldsymbol{\delta}\vec{W}_V = \bar{\vec{V}}_O\vec{S}\bar{\boldsymbol{\Sigma}}_O^{-1}\vec{V}_V^\top
\end{equation*}
where $\vec{W}_O = \bar{\vec{U}}_O\bar{\boldsymbol{\Sigma}}_O\vec{U}_O^\top$ and $\vec{W}_V = \bar{\vec{U}}_O\bar{\boldsymbol{\Sigma}}_V\vec{U}_V^\top$ are SVDs.
\end{proof}

In an approach like that encoded in \texttt{nanoGPT}, $C = D$ and we again have $D^2$ unidentified degrees of freedom. In total, a multi-head attention of that sort with $D = E/H$ (for an $H$ that divides $E$) batches $4E^2 = 4EHD$ parameters for queries, keys, values, and output projection. But $2HD^2 = 2E^2/H$ of the parameters are unidentified, so that we can still only really train a fraction 
\begin{equation*}
	1 - \frac{2E^2/H}{4E^2} 
		= 1 - \frac{1}{2H}
\end{equation*}
of these into real effects. In other words, the overall result about the fraction of meaningful parameters has not changed. 

As an aside, we could consider $H=1$, $D = C = E$ a formulation of a single-head attention, with of course the same outcome about the identified fraction of parameters. 

\subsection{Invariance with Query-Key Bias Terms.}

If we include query-key biases we have to modify our invariance results slightly. Observe
\begin{align*}
	&(\vec{W}_K\vec{X} + \vec{b}_K\vec{1}^\top)^\top (\vec{W}_Q\vec{X} + \vec{b}_Q\vec{1}^\top)
	\\
	&\quad\quad\quad\quad
		= \vec{X}^\top \vec{W}_K^\top \vec{W}_Q \vec{X} 
			+ \vec{X}^\top \vec{W}_K^\top \vec{b}_Q \vec{1}^\top
			+ \vec{1} \vec{b}_K^\top \vec{W}_Q\vec{X}
			+ \big( \vec{b}_K^\top \vec{b}_Q \big) \vec{1} \vec{1}^\top
\end{align*}
The mixtures mean if we perturb $(\vec{W}_Q, \vec{W}_K)$ we also have to perturb $(\vec{b}_Q, \vec{b}_K)$. Specifically consider
\begin{align*}
	( \; (\vec{W}_Q,\vec{b}_Q) \;,\; (\vec{W}_K,\vec{b}_K) \; )
		\to ( \; (\vec{S}\vec{W}_Q,\vec{b}_Q) \;,\; (\vec{S}^{-\top}\vec{W}_K,\vec{b}_K) \; )
\end{align*}
resulting in query-key products
\begin{align*}
	\vec{X}^\top \vec{W}_K^\top \vec{W}_Q \vec{X} 
			+ \vec{X}^\top \vec{W}_K^\top \vec{S}^{-1} \vec{b}_Q \vec{1}^\top
			+ \vec{1} \vec{b}_K^\top \vec{S} \vec{W}_Q \vec{X}
			+ \big( \vec{b}_K^\top \vec{b}_Q \big) \vec{1} \vec{1}^\top
\end{align*}
which is not an invariant result. But
\begin{align*}
	( \; (\vec{W}_Q,\vec{b}_Q) \;,\; (\vec{W}_K,\vec{b}_K) \; )
		\to ( \; ( \vec{S}\vec{W}_Q,\vec{S} \vec{b}_Q ) \;,\; ( \vec{S}^{-\top}\vec{W}_K, \vec{S}^{-\top}\vec{b}_K) \; )
\end{align*}
for which we would have 
\begin{align*}
	&\vec{X}^\top \vec{W}_K^\top \vec{W}_Q \vec{X} 
			+ \vec{X}^\top \vec{W}_K^\top \vec{S}^{-1} (\vec{S} \vec{b}_Q) \vec{1}^\top
			+ \vec{1} (\vec{S}^{-\top} \vec{b}_K)^\top \vec{S} \vec{W}_Q \vec{X}
			+ \vec{1} (\vec{S}^{-\top} \vec{b}_K)^\top (\vec{S} \vec{b}_Q) \vec{1}^\top
	\\
	&\quad\quad\quad\quad
		= \vec{X}^\top \vec{W}_K^\top \vec{W}_Q \vec{X} 
			+ \vec{X}^\top \vec{W}_K^\top (\vec{S}^{-1} \vec{S}) \vec{b}_Q \vec{1}^\top
			+ \vec{1} \vec{b}_K^\top (\vec{S}^{-1} \vec{S}) \vec{W}_Q \vec{X}
			+ \vec{1} \vec{b}_K^\top (\vec{S}^{-1} \vec{S}) \vec{b}_Q \vec{1}^\top
	\\
	&\quad\quad\quad\quad
		= (\vec{W}_K\vec{X} + \vec{b}_K\vec{1}^\top)^\top (\vec{W}_Q\vec{X} + \vec{b}_Q\vec{1}^\top)
\end{align*}

\subsection{Irrelevant Bias Terms.} 

Most of the above has ignored bias terms. \citet{vaswani2023} don't mention bias terms, perhaps suggesting they can be excluded in transformations related to attention. But code like \texttt{nanoGPT} \citep{karpathy2022} includes them, at least optionally, and the review from Google Brain \citep{phuong2022} explicitly admits biases on all terms. Generally speaking tooling like \texttt{pytorch} \citep{paszke2019} is pretty liberal in allowing trainable linear parameters in various operations, in principle admitting identification-confounding linear operations like those discussed above. 

The purpose of this section is to cover some basic results related in character but not substance: bias terms related to the key values are irrelevant in any $\Softmax$-based attention head, and bias related to value projections are irrelevant in any multi-head attention head. While much of the above derives from analysis of the quadratic form implied by query-key products, the $\Softmax$ has it’s own “invariant subspaces”. In particular, a $\Softmax$ is invariant to value changes that are constant over columns with one of the biases we might use having only a constant effect on columns. Multi-head attention introduces another invariance in biases with output pooling. 

The number of parameters we can omit from these observations about biases is small relative to the "second order" invariances in query-key products. We never have to include $HD = E$ (when $C = D$ and $H$ divides $E$) key biases or $HC$ value biases per multi-head attention layer, or when $C = D$ only $2E$ parameters per layer in total. 

\subsubsection{Key Bias Has No Effect.}

As is well known any uniform shift in all exponentiated "logits" doesn't affect a logistic/$\Softmax$: 
\begin{fact}
For any numbers $u_1,\dotsc,u_N$
\begin{equation*}
	\frac{ \exp\{u_n + \delta\} }{ \sum_m \exp\{ u_m + \delta \} }
		= \frac{ e^\delta \exp\{u_n\} }{ \sum_m e^\delta \exp\{ u_m \} }
		= \frac{ e^\delta \exp\{u_n\} }{ e^\delta \sum_m \exp\{ u_m \} }
		= \frac{ \exp\{u_n\} }{ \sum_m \exp\{ u_m \} }
\end{equation*}
\end{fact}
This is actually the basis of a classical computational trick (used in pytorch and most serious implementations of logistics, expit, or softmax) to assert the exponents are never larger than one: just subtract away the maximum value from each value prior to exponentiating to avoid overflows.

Inside a $\SoftmaxCol$ operation we're going to effectively do something like
\begin{equation*}
	\frac{ \exp\{ \vec{k}_d^\top \vec{q}_\ell \} }{ \sum_c \exp\{\vec{k}_c^\top \vec{q}_\ell\} }
\end{equation*}
which will have the same property. We leave off the sum's upper limit to emphasize this could be ``causal'' or not. Again suppose we added query-key biases, 
\begin{align*}
	\vec{K}(\vec{X})^\top \vec{Q}(\vec{X})
		&= \left( \vec{W}_K\vec{X} + \vec{b}_K\vec{1}^\top  \right)^\top \left( \vec{W}_Q\vec{X} + \vec{b}_Q\vec{1}^\top \right)
		\\
		&= \vec{X}^\top \vec{W}_K^\top \vec{W}_Q \vec{X} 
			+ \vec{X}^\top \vec{W}_K^\top \vec{b}_Q \vec{1}^\top
			+ \vec{1} \vec{b}_K^\top \vec{W}_Q \vec{X}
			+ \left( \vec{b}_K^\top \vec{b}_Q \right) \vec{1} \vec{1}^\top
\end{align*}
Note that the last term, $( \vec{b}_K^\top \vec{b}_Q) \vec{1} \vec{1}^\top$, is a constant applied to all rows and columns. Note also the third term, $\vec{1} \vec{b}_K^\top \vec{W}_Q \vec{X}$, is constant over columns but may differ over rows. So when we take $\SoftmaxCol$ over this result, those two terms will be irrelevant (drop out of the exponential fractions), and only the query bias $\vec{b}_Q$ will have an impact on the attention layer output. 

\subsubsection{Value Bias is Unidentified in a Multi-Head Attention}

Suppose we use multi-head attention and include value biases, 
\begin{equation*}
	\vec{W}_V^h \vec{Y}^h + \vec{b}_V^h \vec{q}^\top
	\quad\quad
	\vec{Y}^h(\vec{X}) = \vec{X} \SoftmaxCol\left( \vec{K}^h(\vec{X})^\top\vec{Q}^h(\vec{X}) \right)
\end{equation*}
keeping in mind we will then pool head effects in an affine form
\begin{equation*}
	\sum_{h=1}^H \vec{W}_O^h \left( \vec{W}_V^h \vec{Y}^h(\vec{X}) + \vec{b}_V^h \vec{1}^\top \right) + \vec{b}_O \vec{1}^\top
\end{equation*}
This is obviously expanded to
\begin{equation*}
	\sum_{h=1}^H \vec{W}_O^h \vec{W}_V^h \vec{Y}^h(\vec{X}) + \vec{b}_{V,O} \vec{1}^\top
	\quad\quad
	\vec{b}_{V,O} = \sum_{h=1}^H \vec{W}_O^h \vec{b}_V^h + \vec{b}_O
\end{equation*}
Numerical training will not be able to estimate anything other than $\vec{b}_{V,O}$, not the individual head (or output) biases. In other words, we might as well simple drop the head-value biases and just use $\vec{b}_O$. 

\subsection{Extensions.}

In this short section we discuss a few extensions related to relevant literature.

\subsubsection{Using Larger Heads.}

Suppose, say following \citet{bhojanapalli2020}, we let the per-head outputs be larger than the ``conventional'' $E/H$. We could look at this as a ``representational fan out'' in the heads. The per head identification result about invariance is the same. Along with larger heads we get ``less effect per parameter'' in this sense. So let's take an extreme: $D = E$ for all $H$ heads. This is actually the limiting extreme, as we can't have $D > E$ without purposefully introducing unidentified interactions. Because the fraction of unidentified parameters is like $D/(2E)$, with $1/(2H)$ falling out of the architectural choice $D = E/H$. When $D = E$ this becomes $1/2$ {\em just like a single-head attention}. This potentially exemplifies a balance the literature has not yet observed: with $H = E$ and $D = E/H = 1$ we would have the highest fraction of identified parameters but have the most low rank representations; with $D = E$ and any $H$ we have ``high rank'' representations but also the lowest fraction of identified parameters. Somewhere in between these effects likely balance each other. A coarse review of their models used for ablation studies is in Fig. \ref{fig:bhojanapalli2020-ablation} shows that most of the {\em incremental} benefit happens with more identified parameters, not with significantly larger-than-conventional heads. 

\begin{figure}[h]
    \centering
\begin{tabular}{rcrrrr}
  \hline
  $d_p$ & Total Params & Attn. & Unidentified & $\triangle$ F1 & $\triangle$ EM \\
  \hline
	 32  &  130 &   12.6 &     0.26M (0.20\%) & & \\
	 64  &  142 &   25.2 &     1.05M (0.74\%) & 0.98 & 1.22 \\
	128  &  168 &   50.4 &     4.20M (2.50\%) & 0.09 & 0.32 \\
	256  &  218 &  100.9 &    16.82M (7.71\%) & 0.73 & 0.63 \\
  \hline
\end{tabular}
    \caption{Review of SQuAD ablation studies in \citet{bhojanapalli2020} from our lens of attention parameter identification. Note $d_p$ is their notation, and $d_p = 64$ is the ``conventional'' setting $D = E/H$.}
    \label{fig:bhojanapalli2020-ablation}
\end{figure}

\subsubsection{RoFormer/RoPE.}

RoPE \citep{su2023} is effectively the following modification: 
\begin{equation*}
	( \vec{X}^\top\vec{W}_K^\top\vec{W}_Q\vec{X} \big)_{m,n}
		\mapsto \vec{x}_m^\top\vec{W}_K^\top \vec{R}_{n-m} \vec{W}_Q \vec{x}_n
\end{equation*}
for a family of orthogonal matrices $\{\vec{R}_{n-m}\}$. Of course the specific elements and encoding grounded in $2\times 2$ rotations is critical to the approach. Each sequence pair mapping will appear to have a similar invariance to that discussed above, 
\begin{equation*}
	\vec{W}_K^\top \vec{R}_{n-m} \vec{W}_Q
		\equiv (\vec{S}^{-\top}\vec{R}_{n-m}^\top\vec{W}_K)^\top \vec{R}_{n-m} (\vec{R}_{n-m}^{-1}\vec{S}\vec{W}_Q)
\end{equation*}
but not all valid at once. So we can consider group transformations like $(\vec{W}_Q,\vec{W}_K) \mapsto (\vec{A}\vec{W}_Q,\vec{B}\vec{W}_K)$ for which we must have 
\begin{equation*}
	\vec{W}_K^\top \vec{R}_{n-m} \vec{W}_Q
		= \vec{W}_K^\top \vec{B}^\top \vec{R}_{n-m} \vec{A} \vec{W}_Q
\end{equation*}
and thus $\vec{A} = \vec{R}_{n-m}^{-1} \vec{B}^{-\top} \vec{R}_{n-m}$ for all valid $n,m$. $\vec{A}$ can't depend on $n-m$ itself, and thus $\vec{A}$ (effectively) has to satisfy a set of commutation relationships. 

So, suppose $\vec{S} \in \Rnxm{D}{D}$ commutes with all $\vec{R}_{n-m}$ ($\vec{R}_{n-m} \vec{S} = \vec{S}\vec{R}_{n-m}$); then
\begin{align*}
	(\vec{S}^{-\top}\vec{W}_K)^\top \vec{R}_{n-m} (\vec{S}\vec{W}_Q)
		&= \vec{W}_K^\top \vec{S}^{-1} \vec{R}_{n-m} \vec{S} \vec{W}_Q \\
		&= \vec{W}_K^\top \vec{S}^{-1} \vec{S} \vec{R}_{n-m} \vec{W}_Q \quad \text{(commutativity)} \\
		&= \vec{W}_K^\top \vec{R}_{n-m} \vec{W}_Q
\end{align*}
So our now-valid family of $D \times D$ matrices has additional constraints related to commutativity. For the particular $D/2$ planar rotations $\vec{R}_{n-m}$ in RoPE commuting matrices have two free numbers: any other rotation over the same subspace and a scaling. Handwaving a bit, that means an $\vec{S}$ in an effectively similar $2 \times 2$ block-diagonal form encoded in complex notation as $\alpha_d e^{i\theta_d}$ for $D/2$ scalings $\alpha_d$ and rotations $\theta_d$, having a total of $D$ not $D^2$ degrees of freedom. 

We've effectively shown:
\begin{prop}
Using a RoPE encoding in a transformer we have $D$-dimensional invariance in query-key products ($(\C \setminus \{0\})^{D/2}$ instead of $GL_D(\R)$) of the form 
\begin{equation*}
	\vec{S} = \begin{pmatrix}
		\alpha_1 e^{i\theta_1} & \vec{0} & \dotsb & \vec{0} \\
		\vec{0} & \alpha_2 e^{i\theta_2} & \dotsb & \vec{0} \\
		\vdots & \vdots & \ddots & \vdots \\
		\vec{0} & \vec{0} & \dotsb & \alpha_{D/2} e^{i\theta_{D/2}} \\
	\end{pmatrix}
\end{equation*}
\end{prop}

\subsection{Grouped Query Attention.}

We can sketch GQA as defining $G$ key groups for $H$ heads using
\begin{equation*}
	(\vec{W}_K^{g(h)})^\top\vec{W}_Q^{h}
	\quad\quad
	h = 1,\dotsc,H
	\;\;
	g : \{1,\dotsc,H\} \to \{1,\dotsc,G\}
\end{equation*}
There will thus similar symmetry except every group $G$ has its own. That is, 
\begin{equation*}
	(\vec{W}_K^{g(h)})^\top\vec{W}_Q^{h}
		\equiv ((\vec{S}^{g(h)})^{-\top}\vec{W}_K^{g(h)})^\top\vec{S}^{g(h)}\vec{W}_Q^{h}
\end{equation*}
So we have the same identification properties {\em per group}, not {\em per head}. The identifications are then
\begin{equation*}
	\vec{W}_Q^{h} \mapsto \vec{S}^{g(h)} \vec{W}_Q^{h}
	\quad,\quad
	\vec{W}_K^{g} \mapsto (\vec{S}^{g})^{-\top} \vec{W}_K^{g}
\end{equation*}
for $h = 1,\dotsc,H$ and $g = 1,\dotsc,G$. This is a complete description of the invariance, any other transformation pairs $\{\vec{A}^h\}_h$ and $\{\vec{B}_g\}_g$ would need to satisfy $A_h = B_{g(h)}^{-\top}$. When $G < H$ this is a proper constraint and will decrease the count of unidentified parameters from $H D^2$ to $GD^2$. GQA thus cuts unidentified parameters like $H/G$. 

RoPE reduces symmetry and thus improves identification by a rotation not all transformations commute with, while GQA reduces it by tying heads together so they can no longer transform independently. Some models use both RoPE and GQA meaning group-identified invariance that also commute with the rotations, where the two effects compose: a RoPE + GQA layer has $GD$ unidentified parameters against $H \cdot D^2$ for more basic transformers. Specifically for Llama-2-70B parameters ($H=64, G=8, D=128$) that is factor of $8$ difference: $8 \times 128^2=131,072$ versus $1,048,576$. 

\subsection{Equivalence.}

Returning to multi-head attention (MHA), we could ask the architectural question of whether there is a comparable single-head attention (SHA) from a purely identification standpoint. In other words, given $(E, H, D)$ (fixing $C = D$) are there choices of $(E_0,D_0)$ that have the same ``effective'' parameter counts? Here we will ignore bias for simplicity, and look to equate
\begin{equation*}
	(2D_0+E_0)E_0 - D_0^2 = 4EHD - 2HD^2
\end{equation*}

Consider first the default case where $D = C = E/H$, $H$ divides $E$, and $D_0 = E_0$. Then we effectively parameterize the MHA by $(E, H)$, have $4E^2$ parameters with $2E^2/H$ unidentified. The SHA has $3E_0^2$ parameters with $E_0^2$ unidentified. We can equate these as 
\begin{equation*}
	E_0 = \left\lceil \; E \sqrt{2 - \frac{1}{H}} \; \right\rceil
\end{equation*}
For $H = 1$ we obviously recover SHA, for 
\begin{equation*}
	E_0 = \left\lceil \; E \sqrt{2 - \frac{1}{1}} \; \right\rceil
		= \left\lceil \; E \sqrt{1} \; \right\rceil
		= E
\end{equation*}
For $H > 1$ this also plainly means we need a {\em larger} embedding dimension in SHA than in MHA. While using a larger embedding dimension is possible that will influence the {\em entire network} which we have mostly ignored. In the simplified case the inflation in model size could be large, roughly 60\% larger for $12$ heads. If executed that will result in a false comparison between the MHA-matched SHA because the capacity of other components like the residual skip layers and MLP's in a full transformer will have altered capacity as well. 

Imposing this constraint, $E_0 = E$, more generally we can get a no-go or impossibility result. Set $D_0 = \eta E_0$ to prove this. We would have
\begin{equation*}
	(1 + 2\eta - \eta^2)E^2 = 4EHD - 2HD^2
\end{equation*}
or 
\begin{equation*}
	1 + 2\eta - \eta^2 
		= 2H \left( 2 - \frac{D}{E} \right) \left( \frac{D}{E} \right)
\end{equation*}
When $D = E/H$, $D/E = 1/H$ and
\begin{equation*}
	1 + 2\eta - \eta^2 
		= 2\left( 2 - \frac{1}{H} \right)
\end{equation*}
The LHS is a concave quadratic whose maximum is $2$ (at $\eta = 1$), whereas the RHS takes the value $3$ at $H = 2$ (and increases from there); $H = 1$ again is a special case recovering the obvious equivalence. In other words, there is no comparable SHA for a given MHA with more than one head {\em without} changing, particularly increasing, the embedding dimension. 


\section{Numerical Examples}

We have shown that there are specific “global” as well as locally linearly invariances in attention, it is not clear whether gradient-based updates {\em avoid} these subspaces. After all, shouldn’t a “direction of greatest increase” in some loss translate to implicitly avoiding subspaces over which there is no change? Moreover, the increase in relative identification from using a multi-head attention means more parameters are more useful to predictions. Moreover, can we evaluate how influential the improved identification from multi-head attention alone is, to gauge how much else we should be further ``reading into'' the specific functional form of the attention?

\subsection{Gradient Flow.}

Consider a gradient flow like 
\begin{equation*}
	\frac{d}{dt}(\vec{W}_Q,\vec{W}_K) = - \frac{\partial \mathcal{L}}{\partial (\vec{W}_K,\vec{W}_Q)}
\end{equation*}
for a loss function $\mathcal{L}$. 

\begin{lem}
$\vec{W}_Q\vec{W}_Q^\top - \vec{W}_K\vec{W}_K^\top$ is an invariant under gradient flow for a differentiable loss.
\end{lem}
This also dates back to \citet{arora2018} as ``imbalance'' in a cross-layer setting.
\begin{proof}
Recognizing that the loss only depends on $\vec{W}_K^\top\vec{W}_Q$ and letting $\vec{G} = \partial \mathcal{L} / \partial (\vec{W}_K^\top\vec{W}_Q)$ the chain rule implies
\begin{equation*}
	\frac{\partial \mathcal{L}}{\partial \vec{W}_Q} = \vec{W}_K \vec{G}
	\quad\quad
	\frac{\partial \mathcal{L}}{\partial \vec{W}_K} = \vec{W}_Q \vec{G}^\top
\end{equation*}
Thus
\begin{equation*}
	\frac{d}{dt}(\vec{W}_Q\vec{W}_Q^\top) 
		= - \left( \frac{\partial \mathcal{L}}{\partial \vec{W}_Q} \vec{W}_Q^\top 
			+ \vec{W}_Q \frac{\partial \mathcal{L}}{\partial \vec{W}_Q}^\top \right)
		= - \left( \vec{W}_K\vec{G}\vec{W}_Q^\top + \vec{W}_Q\vec{G}^\top\vec{W}_K^\top \right)
\end{equation*}
and similarly
\begin{equation*}
	\frac{d}{dt}(\vec{W}_K\vec{W}_K^\top) 
		= - \left( \frac{\partial \mathcal{L}}{\partial \vec{W}_K} \vec{W}_K^\top 
			+ \vec{W}_K \frac{\partial \mathcal{L}}{\partial \vec{W}_K}^\top \right)
		= - \left( \vec{W}_Q\vec{G}^\top\vec{W}_K^\top + \vec{W}_K\vec{G}\vec{W}_Q^\top \right)
\end{equation*}
using gradient flow for $d/dt \vec{W}_Q$ and $d/dt \vec{W}_K$. This obviously implies 
\begin{equation*}
	\frac{d}{dt}(\vec{W}_Q\vec{W}_Q^\top) = \frac{d}{dt}(\vec{W}_K\vec{W}_K^\top) 
\end{equation*}
from which the result is simply this as
\begin{equation*}
	\frac{d}{dt}(\vec{W}_Q\vec{W}_Q^\top - \vec{W}_K\vec{W}_K^\top) = \vec{0}.
\end{equation*}
\end{proof}
Ignoring numerical concerns, this means we should be able to quantify the degree to which an actual training sequence ``uses'' unidenfitied parameters via computing $\frac{d}{dt}|| \vec{W}_Q\vec{W}_Q^\top - \vec{W}_K\vec{W}_K^\top ||_F$ so to speak. Using the theoretical gradients this should be zero (the norm should be constant-ish). But adaptive methods can deviate from this ``conservation law''. 

So consider a least squares problem
\begin{equation*}
	\min_{\vec{S} \in \Rnxm{D}{D}}
		\frac{1}{2}
		\Bigg\{
			|| \delta \vec{W}_Q - \vec{S}\vec{W}_Q ||_F^2
			+ || \delta \vec{W}_K - \vec{S}^\top\vec{W}_K ||_F^2
		\Bigg\}
\end{equation*}
Stationarity would require
\begin{equation*}
	\vec{S}\vec{W}_Q\vec{W}_Q^\top + \vec{W}_K\vec{W}_K^\top \vec{S}
		= \left( \delta \vec{W}_Q \vec{W}_Q^\top - \vec{W}_K \delta\vec{W}_K^\top \right)
\end{equation*}
with a computable RHS yielding a Sylvester equation. There is a unique solution when the Gram matrices $\vec{G}_Q = \vec{W}_Q\vec{W}_Q^\top$ and $-\vec{G}_K = -\vec{W}_K\vec{W}_K^\top$ do not share eigenvalues. If we then compute eigendecompositions
\begin{equation*}
	\vec{G}_Q\vec{U}_Q = \vec{U}_Q\boldsymbol{\Lambda}_Q
	\quad\quad
	\vec{G}_K\vec{U}_K = \vec{U}_K\boldsymbol{\Lambda}_K
\end{equation*}
we have some eigenvalues $\lambda_{Q,d},\lambda_{K,d} \geq 0$ for $d=1,\dotsc,D$ from positive semi-definiteness of Gram matrices with strict inequality when the weights are full rank. Then there is a unique $\vec{S}$ when
\begin{equation*}
	\lambda_{Q,d} - (- \lambda_{K,d^\prime}) 
		= \lambda_{Q,d} + \lambda_{K,d^\prime}
		\neq 0
\end{equation*}
for all $d,d^\prime$. This will hold so long as {\em at least one} of the weight matrices is full rank, because then $\min_d \lambda_{Q,d} > 0$ or $\min_d \lambda_{K,d} > 0$. We thus report a value
\begin{equation*}
	\mu = \min_d \lambda_{Q,d} + \min_d \lambda_{K,d}
\end{equation*}
which quantifies rank deficiency that should correlate to {\em invalidating} $D^2$ unidentified degrees of freedom. 

Moreover when $\mu > 0$ we can solve for $\vec{S} = \vec{U}_K \tilde{\vec{S}} \vec{U}_Q^\top$ via
\begin{align*}
	\tilde{\vec{S}} \boldsymbol{\Lambda}_Q
			+ \boldsymbol{\Lambda}_K \tilde{\vec{S}}
		= \vec{R}
	\quad,\quad
		\vec{R} = \vec{U}_K^\top 
			\left( \delta \vec{W}_Q \vec{W}_Q^\top - \vec{W}_K \delta\vec{W}_K^\top \right)
			\vec{U}_Q
\end{align*}
which is a calculation of the RHS followed by elementwise divisions:
\begin{align*}
	\big( \tilde{\vec{S}} \big)_{d,d^\prime}
		= \frac{ ( \vec{R} )_{d,d^\prime} }{ \lambda_{Q,d} + \lambda_{K,d^\prime} }
\end{align*}
Thus we have a formal solution $\vec{S}$ to the least squares problem, a measure
\begin{align*}
	\rho = \frac{ 
		|| \vec{S} \vec{W}_Q ||_F^2 + || \vec{S}^\top \vec{W}_K ||_F^2
	 }{ 
		|| \delta\vec{W}_Q ||_F^2 + || \delta\vec{W}_K ||_F^2
	}
	\leq 1
\end{align*}
that describes a normalized projection length, and even an update modification if we were interested. 

Note that eigendecompositions of the Gram matrices are relatively feasible being sized by $D$. In our specific numerical examples from training we have $D = 64$ for which these operations are trivial on modern computers. 

We will also report 
\begin{align*}
	\beta = || \vec{W}_Q\vec{W}_Q^\top - \vec{W}_K\vec{W}_K^\top ||_F
\end{align*}
which over a training run measures how conserved the conserved quantity is and
\begin{align*}
	\tau = \mathrm{trace}\big( \vec{W}_Q\vec{W}_Q^\top - \vec{W}_K\vec{W}_K^\top \big)
\end{align*}
which, alongside $\beta$, shows if deviation from conservation is signed. 

\subsection{Training Runs.}

We run training experiments on relatively ``stock'' \texttt{nanoGPT} \citep{karpathy2022} with the tiny shakespeare dataset using character tokens. This is to avoid questions that could arise from a completely custom implementation as well as include effects like dropout to show their influence (or not). Scale should not be an important factor for empirical validation of our results, so small scale tests are enough.

\subsubsection{Training Cases.}

We run four cases to empirically explore our theoretical results:
\begin{itemize}
\item[(1)] SGD with no weight decay and a learning rate of 0.25
\item[(2)] SGD with small weight decay $4\times 10^{-4}$ and a learning rate of 0.25 (to roughly match case 3)
\item[(3)] AdamW (the default) with weight decay 0.1 and a learning rate of $1\times 10^{-3}$
\item[(4)] AdamW with weight decay 0.1 and a learning rate of $1\times 10^{-3}$ but ``rebalancing'' weights every 200 iterations.
\end{itemize}
In each case we train for 5,000 steps and store sets of metrics every 10 steps including generic training metrics, step-specific bias gradient norms, and the query/key weight metrics discussed above ($\rho$, $\beta$, $\tau$, and $\mu$). Case (1) is worth including because it should run entirely orthogonally to unidentified parameter subspaces, by definition of the gradient. However \texttt{nanoGPT} will include numerical errors, dropout, and additional features that could break this result. Case (2) adds weight decay which formally adds a diagonal multiplier into weight updates we have not analyzed. Case (3) should have the best training performance but also has the highest potential to make changes changing unidentified parameters. Case (4) is an experiment to uncover if ``fixing'' AdamW step overlap with unidentified parameter subspaces has any effect on training. In principle the rebalancing we use destroys the higher order moment data so some degradation in training performance should be expected though overall our theory predicts that the effect should be minor.

We also run a sweep of case (1) with various learning rates to sketch out the influence of the related numerical discretization error that would occur. {\em A priori} we would expect larger steps to deviate further from idealistic gradient flow conservation roughly like the learning rate squared and can review results from that point of view. We run this sweep with learning rates $0.5$, $0.25$, $0.125$, and $0.0625$. 

\subsubsection{Rebalancing.}

We should clarify though what our rebalancing step really is to avoid confusion. Note that we could consider this a form of ``quotienting out'' the symmetry, which has a precedent in the literature \citep{zhao2025}. Our process is as follows: When rebalancing we reset 
\begin{align*}
	\vec{M} 
		&= (\vec{W}_K^{(h)})^\top \vec{W}_Q^{(h)}
		= \vec{Q}_K^{(h)} \vec{R}_K^{(h)} (\vec{R}_Q^{(h)})^\top (\vec{Q}_Q^{(h)})^\top
		\\
		&= \vec{Q}_K^{(h)} \vec{U}_M \boldsymbol{\Sigma} \vec{V}_M^\top (\vec{Q}_Q^{(h)})^\top
		= \Big( \boldsymbol{\Sigma}^{1/2} (\vec{Q}_K^{(h)})^\top \vec{U}_M^\top \Big)^\top 
			\Big( \boldsymbol{\Sigma}^{1/2} \vec{V}_M^\top (\vec{Q}_Q^{(h)})^\top \Big)
\end{align*}
taking thin QR decompositions of $\vec{W}_Q^{(h)},\vec{W}_K^{(h)}$ and an SVD for $\vec{R}_K^{(h)} (\vec{R}_Q^{(h)})^\top$. That gives weight resets 
\begin{align*}
	\bar{\vec{W}}_Q^{(h)} = \boldsymbol{\Sigma}^{1/2} \vec{V}_M^\top (\vec{Q}_Q^{(h)})^\top
	\quad\quad
	\bar{\vec{W}}_K^{(h)} = \boldsymbol{\Sigma}^{1/2} (\vec{Q}_K^{(h)})^\top \vec{U}_M^\top
\end{align*}
which provably zeros out the conserved quantity from gradient flow because
\begin{align*}
	\bar{\vec{W}}_Q^{(h)}(\bar{\vec{W}}_Q^{(h)})^\top 
		&= \boldsymbol{\Sigma}^{1/2} \vec{V}_M^\top (\vec{Q}_Q^{(h)})^\top 
				\vec{Q}_Q^{(h)} \vec{V}_M \boldsymbol{\Sigma}^{1/2}
		= \boldsymbol{\Sigma}
	\\
	\bar{\vec{W}}_K^{(h)}(\bar{\vec{W}}_K^{(h)})^\top
		&= \boldsymbol{\Sigma}^{1/2} (\vec{Q}_K^{(h)})^\top \vec{U}_M^\top
				\vec{U}_M \vec{Q}_K^{(h)} \boldsymbol{\Sigma}^{1/2}
		= \boldsymbol{\Sigma}
\end{align*}
Because there are still invariant rotations and permutations (we don't solve for a projection) we also solve
\begin{align*}
	\min_{\vec{Z}^\top\vec{Z}=\vec{I}}
	\; || \vec{Z}\bar{\vec{W}}_Q^{(h)} - \vec{W}_Q^{(h)} ||_F^2
	+ || \vec{Z}\bar{\vec{W}}_K^{(h)} - \vec{W}_K^{(h)} ||_F^2
\end{align*}
with another SVD to get a {\em specific} choice of the update closest to pre-rebalanced step. 

\subsection{Training Results.}

A summary of the results for a single representative run is presented in Fig. \ref{fig:shakespeare-results}. 

\begin{figure}[h]
    \centering
    \includegraphics[width=1.0\textwidth]{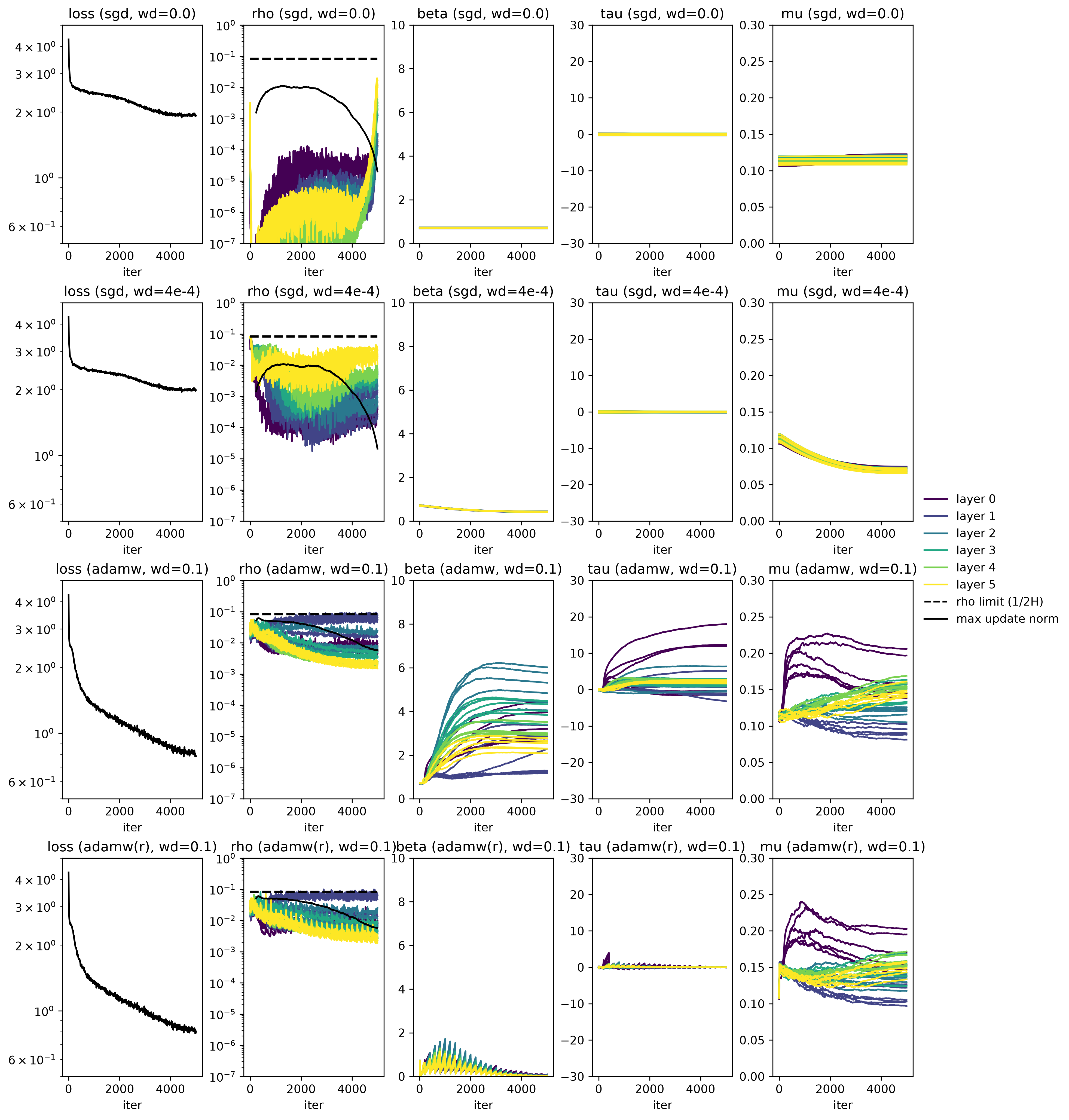}
    \caption{Numerical results from 5000 training iterations on the "tiny shakespeare" datasets using character tokens with \texttt{nanoGPT}. First row: using plain SGD with no weight decay and a 0.25 base learning rate, where unidentified subspaces should be entirely avoided outside of numerical error. Second row: using plain SGD with 0.0004 weight decay matched to a base learning rate 0.25, exhibiting a compressive effect on the query-key trace imbalance discussed in the text. Third row: using AdamW with weight decay, showing in-training effects on both $\beta$ and the trace imbalance from training steps overlapping unidentified subspaces. Even so, AdamW makes the most progress in reducing loss over the iteration range. Fourth row: AdamW with rebalancing.}
    \label{fig:shakespeare-results}
\end{figure}

Note that, in the second column, $\rho$ is plotted along with $1/(2H)$ and the max update norm. The closer $\rho$ is to $1/(2H)$ the more ``unidentified'' actual steps are. We can clearly see pure SGD steps far away from that boundary, decaying weight steps become closer, and AdamW is closest to the boundary with some layer head updates effectively covering the unidentified invariant space. We also show the update norm to clarify ``spikes'' in $\rho$ towards the end of training. The smaller the update norm, the more ill-posed computing $\rho$ becomes, and thus these spikes are explained as a computational artifact. See also Fig. \ref{fig:shakespeare-sweep}. 

The third column shows $\beta$, which should be conserved for gradient flow whose orbits are orthogonal to symmetries creating unidentified parameters. SGD keeps $\beta$ constant (nonzero simply from initialization), SGD with weight decay slowly decreases $\beta$, while AdamW clearly does not keep $\beta$ conserved.

The fourth column shows $\tau$ which is related to a signed $\beta$. For SGD with and without weight decay this is effectively zero, suggestive of a mean-zero initial state for the conserved difference in Gram matrices. The story is entirely different for AdamW. $\tau$ spreads over training, in both positive and negative directions (sign related to queries vs keys), with the largest changes in the earliest layer. 

The fifth column with $\mu$ exists just to show we aren't seeing rank deficiency in these training runs. Rank deficiency would appear as lines dropping to zero, which does not occur. Across all runs this is consistent with our results showing the existence of unidentified parameters when {\em not} paired with rank deficiency. Again $\mu$ is stable for SGD, while AdamW is showing significantly different behavior for this metric of the spectrum of the query and key weights. 

The rebalanced AdamW in the last row proves out the idea that a ``proximal orbit correction'' during training to bring iterates back to a canonical set of weights with respect to attention weight symmetries has no deleterious effect on training. Comparing ``normal'' AdamW to periodic corrections does not change the overall loss but (by construction) regularly zeros out the conserved quantities derived from the unidentified parameter symmetries. Specifically $\beta$ exhibits a sawtooth pattern around the rebalances (with a decaying envelope related to the learning rate) and $\tau$, the signed deviation from conservation using $\mathrm{trace}$, looks much more like plain SGD. As an aside, we ran but do not display the rebalanced AdamW training with optimizer clearing but no weight changes and recovered effectively ``plain'' AdamW behavior. That is, the change in behavior is from the rebalancing, not optimizer state clearing.

Fig. \ref{fig:shakespeare-bias} supports our theoretical results about bias terms. For simplicity, we focus on the query-key biases, as the value-output claims in our theory are both harder to compute (coupled together by weights with output bias) and also may be influenced by dropout in training (while the claims are unaffected by that in forward inference). The conclusion from this figure is simply confirmation: regardless of the optimizer approach the key layer bias has numerically zero gradients, while the query layer bias has small but not likely numerically zero gradients. This is exactly what we claimed to prove.

\begin{figure}[h]
    \centering
    \includegraphics[width=1.0\textwidth]{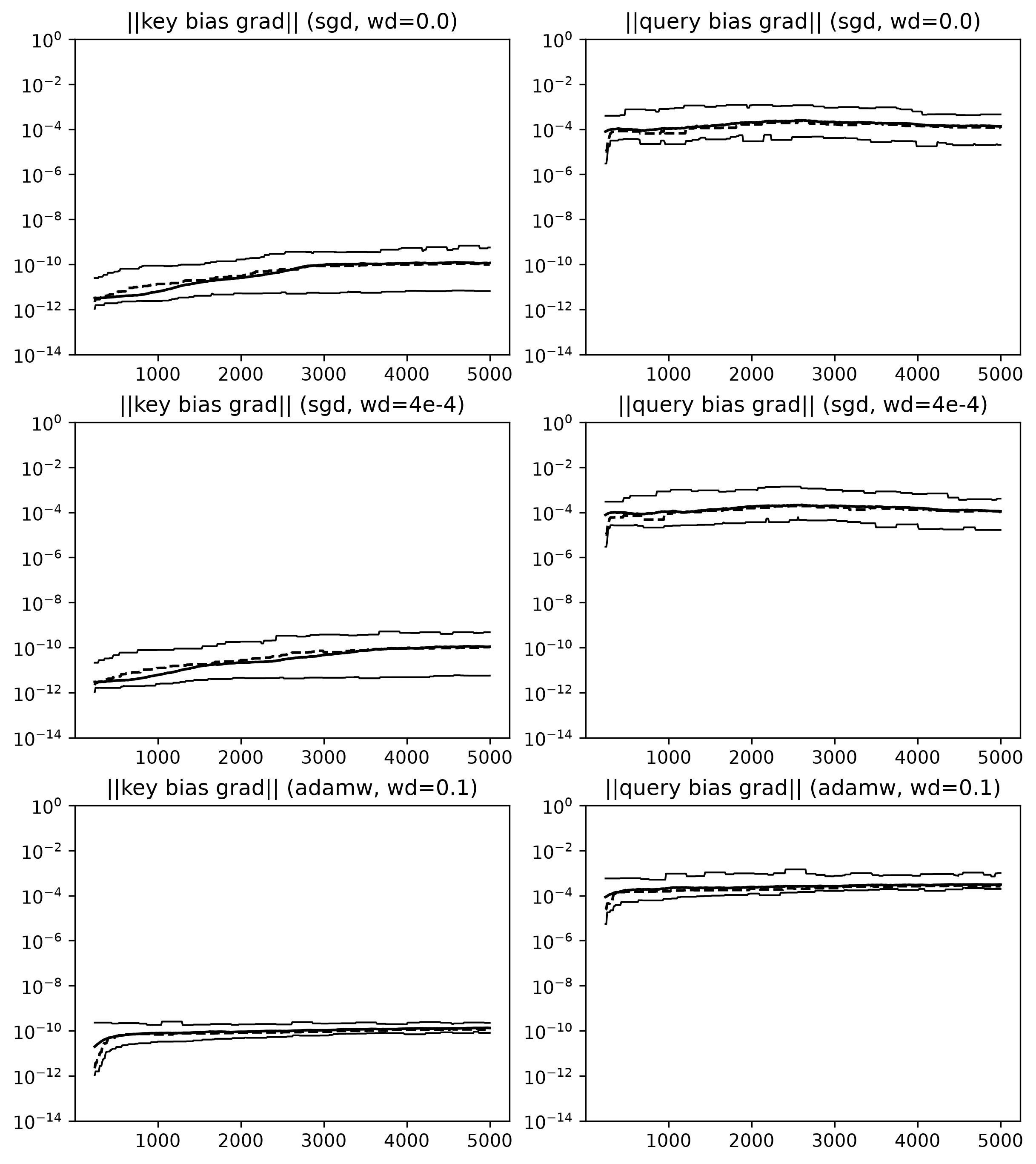}
    \caption{Envelopes (min, max, median, and mean) of the key and query bias for three training runs as in Fig. \ref{fig:shakespeare-results}. Optimizer or settings don't significantly effect the bias gradient norms. Moreover the key bias norm is numerically zero, as required by our theoretical result.}
    \label{fig:shakespeare-bias}
\end{figure}

Results of our sweep of SGD without weight decay at descending learning rates are shown in Fig. \ref{fig:shakespeare-sweep}. One effect we can again see is the spiking of $\rho$ with smaller updates. Another is that ``drift'', the norm of the step-wise difference in the conserved quantity, is larger for earlier layers, stabilizes over training, and for all layers/heads decreases with learning rate. This suggests a story about our simplest optimizer in which discretization error from the learning rate itself is driving lack of conservation. 

\begin{figure}[h]
    \centering
    \includegraphics[width=1.0\textwidth]{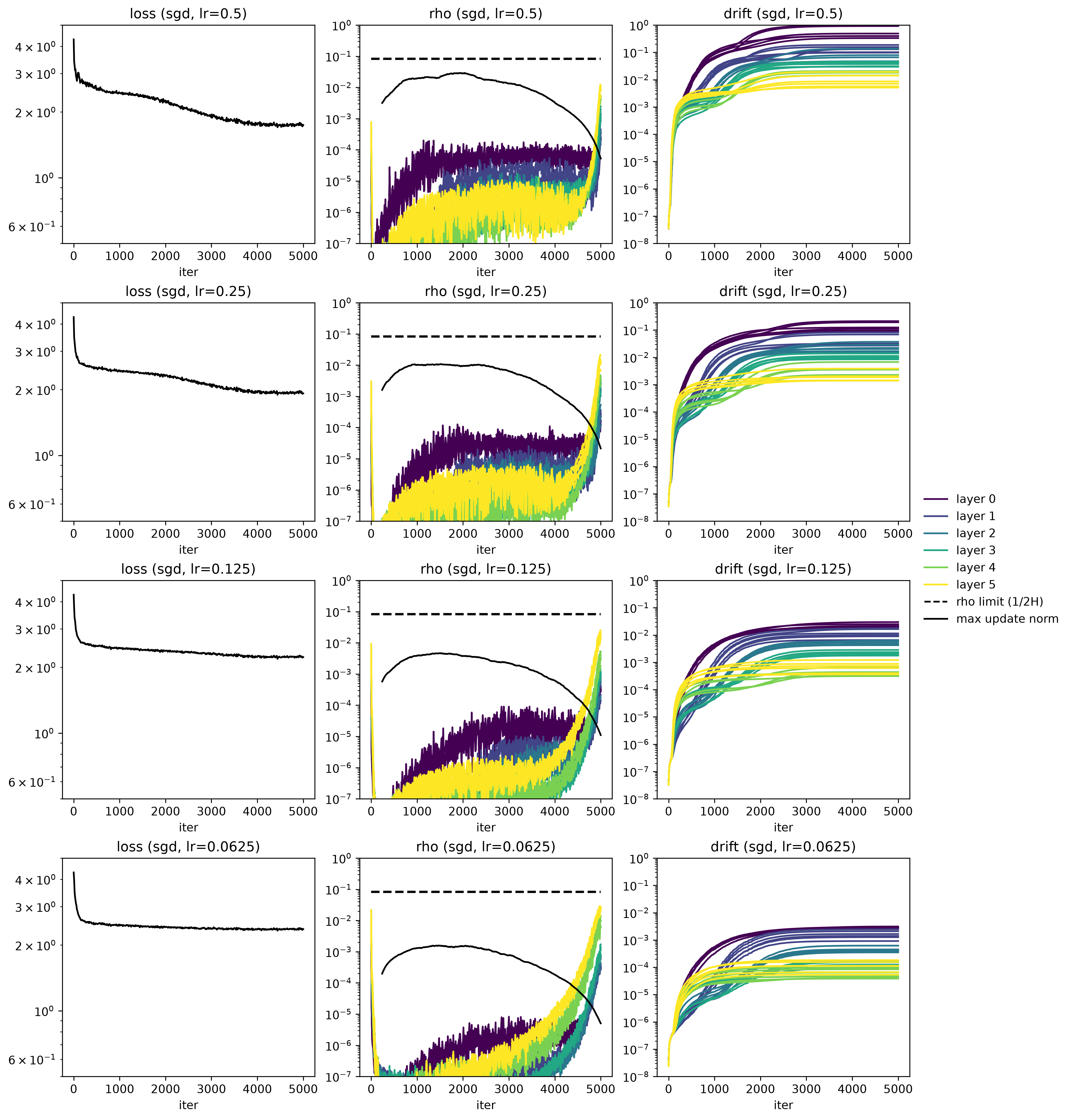}
    \caption{Results for the SGD (no decay) sweep of training runs with descending learning rates. Drift here refers to the conserved quantity drift over a step. There is a steady decrease from higher learning rates to smaller learning rates suggesting influence of discretization error.}
    \label{fig:shakespeare-sweep}
\end{figure}


\section{Conclusions}

The original \citet{vaswani2023} paper introduced transformers and specifically motivated the introduction of multi-head attention as a way to attend to distinct representational subspaces. In this note we have proved, both theoretically and empirically (albeit on a popular small somewhat toy implementation), that introducing additional heads also takes a transformer architecture from $1/2$ of parameters unidentified to $1/(2H)$ (50\% to 80-90\% identified in published architectures). We have also shown matching the identification of the multi-head architecture {\em is not possible} in a naive single-head design. Part of our framing in the latter was to rule out an experiment we could otherwise do, but it has the side-effect of implying that the architectural choice is vital to obtaining the property.

We've generally motivated our work generally by stating ``identification matters'' which is easier said than done at the level of complexity in useful neural networks. However reflecting on the mathematics of attention based layer designs proves this out implicitly for architectures adopted for reasons other than purely statistical well-posedness with identification. Our main result is that using multiple heads implicitly improves identification, but was considered for representationality. RoPE improves it further via commutativity constraints but intends to impose pre-defined structure on the attended tokens suitable for language. GQA is presumably somewhere in between. In other words, while literature and practice has been innovating architectures to change how transformer networks build and weight features they have also changed how many network parameters have ``meaning'' in a basic statistical sense. This is a basic and seemingly unacknowledged confound for understanding performance improvement.

In the process this note has provided buttressing evidence for some other aspects of the study of transformers. We have strengthened the dimensionality of invariance \citep{zhang2025}, shown that in the specific case of multi-head attention the architectural influence on identification is nontrivial, and that the conventional choice ($D = E/H$ where $H$ divides $E$) is likely a very good tradeoff between representationality, efficiency, and identification \citep{bhojanapalli2020}. 


\section*{Declaration of Generative AI and AI-assisted technologies in the writing process}

In the preparation of this note the author(s) used Anthropic's Claude (Opus 5, medium thinking) to prepare a formal referee-style review, suggest targeted related literature, provide suggestions for metrics and plots in the numerical experiements, and simplify the coding process for the numerical experiments. No direct text written by Claude was used in this note, and after using this tool, the author reviewed, edited, and re-derived the content as needed for both substance and style and takes full responsibility for the content.

\printbibliography[title=References]

\end{document}